\documentclass[letterpaper]{article}

\newif\ifdraft

\usepackage{jair}
\usepackage{times}
\usepackage{helvet}
\usepackage{courier}
\usepackage{fixltx2e} 
\usepackage{url}

\usepackage{comment}
\usepackage{paralist}
\usepackage{enumerate}
\usepackage{amsmath}
\usepackage{boxedminipage}

\usepackage{amsthm}

\newtheorem{theorem}{Theorem}

\newtheorem{definition}{Definition}
\newtheorem{example}{Example}

\usepackage{sty-gen-environments}
\usepackage{sty-gen-math-letters}
\usepackage{sty-gen-dl}
\usepackage{sty-main}

\begin{document}

\title{Managing Change in Graph-structured Data Using Description Logics
(long version with appendix)
   \thanks{This paper is a extended version of \cite{acos-aaai14} that
     contains an appendix with proofs.}
}

\author{\name Shqiponja Ahmetaj
 \email  ahmetaj@dbai.tuwien.ac.at \\
 \addr Vienna University of Technology, Austria
\AND
\name Diego Calvanese
 \email  calvanese@inf.unibz.it \\
 \addr Free University of Bozen-Bolzano, Italy
 \AND
 \name Magdalena Ortiz
 \email ortiz@kr.tuwien.ac.at \\
 \addr Vienna University of Technology, Austria
 \AND
 \name Mantas \v{S}imkus
 \email simkus@dbai.tuwien.ac.at \\
 \addr Vienna University of Technology, Austria
}

\maketitle

\begin{abstract}
  In this paper,
  we consider the setting of graph-structured data that evolves as a result of
  operations carried out by users or applications.  We study different
  reasoning problems, which range from ensuring the satisfaction of a given set
  of integrity constraints after a given sequence of updates, to deciding the
  (non-)existence of a sequence of actions that would take the data to an
  (un)desirable state, starting either from a specific data instance or from an
  incomplete description of it.  We consider
  an action language in which actions are finite sequences of conditional
  insertions and deletions of nodes and labels, and use Description Logics for
  describing integrity constraints and (partial) states of the data.  We then
  formalize the above data management problems as a static verification problem
  and several planning problems.  We provide algorithms and tight complexity
  bounds for the formalized problems, both for an expressive DL and for a
  variant of DL-Lite.
\end{abstract}

\section{Introduction}
\label{sec:introduction}

The complex structure and increasing size of information that has to be managed
in today's applications calls for flexible mechanisms for storing such
information, making it easily and efficiently accessible, and facilitating its
change and evolution over time.  The paradigm of \emph{graph structured data}
(GSD) \cite{SaPa11}
has gained popularity recently\footnote{Graph structured data models have their
 roots in work done in the early '90s, see, e.g., \cite{CoMe90}.} as an
alternative to traditional relational DBs that provides more flexibility and
thus can overcome the limitations of an a priori imposed rigid structure on the
data.  Indeed, differently from relational data, GSD do not require a schema to
be fixed a priori.  This flexibility makes them well suited for many emerging
application areas such as managing Web data, information integration,
persistent storage in object-oriented software development, or management of
scientific data.  Concrete examples of models for GSD are RDFS~\cite{W3Crec-RDF-Schema},
object-oriented data models, and XML.

In GSD, information is represented by means of a node and edge labeled graph, in
which the labels convey semantic information.  The representation structures
underlying many standard knowledge representation formalisms, and in particular
Description Logics (DLs) \cite{BCMNP03} are paradigmatic examples of GSD.
Indeed, in DLs
the domain of interest is modeled by means of unary relations (a.k.a.\
\emph{concepts}) and binary relations (a.k.a.\ \emph{roles}), and hence the
first-order interpretations of a DL knowledge base (KB) can be viewed as node
and edge labeled graphs.  DLs have been advocated as a proper tool for data
management \cite{Lenz11}, and are very natural for describing complex
knowledge about domains represented as GSD.  A DL KB comprises an
assertional component, called \emph{ABox}, which is often viewed as a
possibly incomplete instance of
GSD, and a logical theory called terminology or \emph{TBox}, which can be used
to infer implicit information from the assertions in the ABox.  An alternative
possibility is to view the \emph{finite} structures over which DLs are
interpreted as (complete) GSD, and the KB as a description of constraints and
properties of the data.  Taking this view, DLs have been applied, for
example, for the static analysis of traditional data models, such as UML class
diagrams \cite{BeCD05} and Entity Relationship schemata \cite{ACKRZ07b}.
Problems such as the consistency of a diagram are reduced to KB satisfiability
in a suitable DL, and DL reasoning services become tools for managing GSD.

In this paper, we follow the latter view, but aim at using DLs not
only for static reasoning about data models, but also for reasoning
about the evolution and change over time of GSD that happens as the result of
executing actions. The development of automated tools to
support such tasks is becoming a pressing problem, given the large
amounts and complexity of GSD currently available. 
Having tools to
understand the properties and effects of actions is important and
provides added value for many purposes, including application
development, integrity preservation, security, and optimization.
Questions of interest are, e.g.:
\begin{asparaitem}
\item Will the execution of a given action \emph{preserve} the integrity
  constraints,  for every initial data instance?
\item Is there a sequence of actions that leads a given data instance into
  a state where some property (either desired or not) holds?
\item Does a given sequence of actions lead every possible initial
  data instance into a state where some property necessarily holds?
\end{asparaitem}

The first question is analogous to a classic problem in relational databases:
verifying \emph{consistency} of database transactions.  The second and third
questions are classic questions in AI (called \emph{planning} and
\emph{projection}, respectively).

In this paper we address these and other related questions, develop tools to
answer them, and characterize the computational properties of the underlying
problems.  The role of DLs in our setting is manifold, and we propose a very
expressive DL that is suitable for:
\begin{inparaenum}[\it (i)]
\item modeling sophisticated domain knowledge,
\item specifying conditions on the state that should be reached (goal state),
  and
\item specifying actions to evolve GSD over time.
\end{inparaenum}
For the latter, we introduce a simple yet powerful language in which actions
are finite sequences of (possibly conditional) insertions and deletions
performed on concepts and roles, using complex DL concepts and roles as
queries. Our results are quite general and allow for analyzing data evolution
in several practically relevant settings, including RDF data under constraints
expressed in RDFS or OWL. Via the standard reification technique~\cite{BeCD05},
they also apply to the more traditional setting of relational data under
schemas expressed in conceptual models (e.g., ER schemas, or UML class
diagrams), or to object-oriented data.

In this setting, we address first the \emph{static verification
  problem}, that is, the problem of verifying whether for every
possible state satisfying a given set of constraints (\ie a given KB), the
constraints are
still satisfied in the state resulting from the execution of a given
(complex) action.  We develop a novel technique similar in spirit to
\emph{regression} in reasoning about actions \cite{LRLLS97}, and are
able to show that static verification is decidable.  We provide tight
complexity bounds for it, using two different DLs as domain languages.
Specifically, we provide a tight \textsc{coNExpTime} bound for the
considered expressive DL, and a tight co\np bound for a variation of
\textit{DL-Lite} \cite{CDLLR07}.
For our setting, we then study different variants of planning.  We define a
plan as a sequence of actions that leads a given structure into a state where
some property (either desired or not) holds. Then we study problems such as
deciding the existence of a plan, both for the case where the initial structure
is fully known, and where only a partial description of it is available, and
deciding whether a given sequence of actions is always a plan for some goal.
Since the existence of a plan (of unbounded length) is undecidable in general,
even for lightweight DLs and resctricted actions, we also study 
plans of bounded length. We provide tight complexity bounds for the different
considered variants of the problem, both for lightweight and for expressive
DLs.  This paper adds an appendix with proofs to \cite{acos-aaai14},
some of the results were published in preliminary form \cite{CaOS13}.

\section{An Expressive DL for Modeling GSD}
\label{sec:preliminaries}

We now define the DL \OURDL, used to express constraints on GSD.  It extends
the standard \ALCHOIQ with Boolean combinations of axioms, a constructor for a
singleton role, union, difference and restrictions of roles, and variables as
place-holders for individuals. The importance of these constructors will be
become clear in Sections~\ref{sec:actions} and~\ref{sec:capturing}.

We assume countably infinite sets $\rolenames$ of \emph{role names},
$\conceptnames$ of \emph{concept names}, $\indivnames$ of \emph{individual
 names}, and $\varnames$ of \emph{variables}. \emph{Roles} are defined
inductively:
\begin{inparaenum}[\it (i)]
\item if $p\in \rolenames$, then $p$ and $p^-$ (the \emph{inverse} of $p$) are
  roles;
\item if $\{t,t'\}\subseteq \indivnames\cup \varnames$, then $\{(t_1,t_2)\}$ is
  also a role;
\item if $r_1,r_2$ are roles, then 
  $r_1 \ROR r_2 $, and $r_1 \RDIFF r_2$ are also roles; and
\item if $r$ is a role and $C$ is a concept, then $r|_C$ is a role.
\end{inparaenum}
\emph{Concepts} are defined inductively as well:
\begin{inparaenum}[\it (i)]
\item if $A\in \conceptnames$, then $A$ is a concept;
\item if $t\in \indivnames\cup \varnames$, then $\{t\}$ is a concept
  (called \emph{nominal});
\item if $C_1$, $C_2$ are concepts, then $C_1 \AND C_2$, $C_1 \OR C_2$, and
  $\NOT C_1$ are also concepts;
\item if $r$ is a role, $C$ is a concept, and $n$ is a non-negative
  integer, then $\SOME{r}{C}$, $\ALL{r}{C}$, $\Atmostq{n}{r}{C}$, and
  $\Atleastq{n}{r}{C}$ are also concepts.
\end{inparaenum}

A \emph{concept} (resp., \emph{role}) \emph{inclusion} is an expression of the
form $\alpha_1\ISA \alpha_2$, where $\alpha_1,\alpha_2$ are concepts (resp.,
roles).  Expressions of the form $t:C$ and $(t,t'):r$, where $\{t,t'\}\subseteq
\indivnames\cup \varnames$, $C$ is a concept, and $r$ is a role, are called
\emph{concept assertions} and \emph{role assertions}, respectively. Concepts,
roles, inclusions, and assertions that have no variables are called
\emph{ordinary}.
We define (\OURDL-)\emph{formulae} inductively:
\begin{inparaenum}[\it (i)]
\item every inclusion and every assertion is a formula;
\item if $\K_1$, $\K_2$ are formulae, so are $\K_1\land\K_2$, $\K_1\lor\K_2$,
  and $\kbneg \K_1$.
\end{inparaenum}
A formula $\K$ with no variables is called \emph{knowledge base (KB)}.


As usual in DLs, the semantics is given in terms of interpretations.  An
\emph{interpretation} is a pair $\I=\tuple{\dom{\I},\Int{\I}{\cdot}}$ where
$\dom{\I} \neq \emptyset$ is the \emph{domain}, $\Int{\I}{A} \subseteq
\dom{\I}$ for each $A\in\conceptnames$,
$\Int{\I}{r}\subseteq\dom{\I}\times\dom{\I}$ for each $r\in\rolenames$, and
$\Int{\I}{o}\in \dom{\I}$ for each $o\in \indivnames$.  For the ordinary roles
of the form $\{(o_1,o_2)\}$, we let
$\Int{\I}{\{(o_1,o_2)\}}=\{(\Int{\I}{o_1},\Int{\I}{o_2})\}$, and for ordinary
roles of the form $r|_{C}$, we let $\INT{\I}{r|_{C}} = \{(e_1,e_2) \mid
(e_1,e_2) \in \Int{\I}{r} \text{ and } e_2 \in \Int{\I}{C} \}$. The function
$\Int{\I}{\cdot}$ is extended to the remaining ordinary concepts and roles in
the usual way, see \cite{BCMNP03}.
Assume an interpretation $\I$.  For an ordinary inclusion $\alpha_1\ISA
\alpha_2$, $\I$ \emph{satisfies} $\alpha_1\ISA \alpha_2$ (in symbols,
$\I\models \alpha_1\ISA \alpha_2$) if $\alpha_1^{\I}\subseteq \alpha_2^{\I}$.
For an ordinary assertion $\beta=o:C$ (resp., $\beta=(o_1,o_2):r$), $\I$
\emph{satisfies} $\beta $ (in symbols, $\I\models \beta$) if $o^{\I}\in C^{\I}$
(resp., $(o_1^{\I},o_2^{\I})\in r^{\I}$).  The notion of satisfaction is
extended to knowledge bases as follows:
\begin{inparaenum}[\it (i)]
\item $\I\models \K_1\land \K_2$ if $\I\models \K_1$ and $\I\models \K_2$;
\item $\I\models \K_1\lor \K_2$ if $\I\models \K_1$ or $\I\models \K_2$;
\item $\I\models \kbneg \K$ if $\I\not \models \K$.
\end{inparaenum}
If $\I\models \K$, then $\I$ is a \emph{model} of $\K$.  The \emph{finite
 satisfiability} (resp., \emph{unsatisfiability}) \emph{problem} is to decide
given a KB $\K$ if there exists (resp., doesn't exist) a model $\I$ of $\K$
with $\dom{\I}$ finite.

A \nexptime lower bound for finite satisfiability in \OURDL follows from the
work of Tobies \cite{Tobies00}.  Using well-known techniques due to
Borgida \cite{Borg96}, a matching upper bound can be shown by a direct
translation into the two variable fragment with counting, for which finite
satisfiability is in \nexptime \cite{Pratt-Hartmann05}. Hence, the finite
satisfiability problem for \OURDL KBs has the same computational complexity as
for the standard \ALCHOIQ:

\begin{theorem}
  \label{thm:dl-sat}
  Finite satisfiability of \OURDL KBs is \nexptime-complete.
\end{theorem}

We are interested in the problem of \emph{effectively} managing GSD satisfying
the knowledge represented in a DL KB $\K$.  Hence, we must assume that such
data are of \emph{finite} size, i.e., they correspond naturally to \emph{finite
 interpretations that satisfy the constraints} in $\K$.  In other words, we
consider configurations of the GSD that are finite models of $\K$.

\section{Updating Graph Structured Data}
\label{sec:actions}

We now define an action language for manipulating GSD, i.e., finite
interpretations. The basic actions allow one to insert or delete individuals
from extensions of concepts, and pairs of individuals from extensions of roles.
The candidates for additions and deletions are instances of complex concepts
and roles. Since our DL supports nominals $\{o\}$ and singleton roles
$\{(o,o')\}$, actions can be defined to add/remove a single individual to/from
a concept, or a pair of individuals to/from a role.
We allow also for action composition and conditional actions.  Note that the
action language introduced here is a slight generalization of the one in
\cite{CaOS13}.

\begin{definition}[Action language]
  A \emph{basic action} $\beta$ is defined by the following grammar:
  \[
    \beta ~\longrightarrow~ (A \oplus C) ~\mid~ (A \ominus C) ~\mid~
    (p \oplus r) ~\mid~ (p \ominus r),
  \]
  where $A$ is a concept name,  $C$ is an arbitrary concept, $p$ is a role
  name, and $r$ is an arbitrary role.
  Then \emph{(complex) actions} are given by
  the following grammar:
  \[
    \alpha ~\longrightarrow~ \varepsilon ~ \mid ~
                             \beta \cdot \alpha ~\mid~
                             (\cact{\K}{\alpha}{\alpha}) \cdot \alpha
  \]
  where $\beta$ is a basic action, $\K$ is an arbitrary \OURDL-formula,
  and $\varepsilon$ denotes the \emph{empty action}.

  A \emph{substitution} is a function $\sigma$ from $\varnames$ to
  $\indivnames$. For a formula, an action or an action sequence $\Gamma$, we use
  $\sigma(\Gamma)$ to denote the result of replacing in $\Gamma$ every
  occurrence of a variable $x$ by the individual $\sigma(x)$. An action
  $\alpha$ is \emph{ground} if it has no variables. An action
  $\alpha'$ is called a \emph{ground instance} of an action $\alpha$
  if $\alpha'=\sigma(\alpha)$ for some substitution $\sigma$. 
\end{definition}

Intuitively, an application of an action $(A \oplus C)$ on an interpretation
$\I$ stands for the addition of the content of $\Int{\I}{C}$ to $\Int{\I}{A}$.
Similarly, $(A \ominus C)$ stands for the removal of $\Int{\I}{C}$ from
$\Int{\I}{A}$.  The two operations can also be performed on extensions of
roles.  Composition stands for successive action execution, and a conditional
action $\cact{\K}{\alpha_1}{\alpha_2}$ expresses that $\alpha_1$ is executed if
the interpretation is a model of $\K$, and $\alpha_2$ is executed otherwise.
If $\alpha_2 = \varepsilon$ then we have an action with a simple
\emph{pre-condition} as in classical planning languages, and we write it
as $\cactne{\K}{\alpha_1}$, omitting $\alpha_2$.

To formally define the semantics of actions, we first introduce the notion of
\emph{interpretation update}.

\begin{definition}[Interpretation update]
  Assume an interpretation $\I$ and let $E$ be a concept or role name. If $E$
  is a concept, let $W\subseteq \dom{\I}$,
  otherwise, if $E$ is a role, let $W\subseteq \dom{\I}\times \dom{\I}$.
  Then, $\I\oplus_{E} W $ (resp., $\I\ominus_{E} W $) denotes the
  interpretation $\I'$ such that $\dom{\I'}=\dom{\I}$, and
  \begin{compactenum}[-]
  \item $E^{\I'}=E^{\I}\cup W$ (resp., $E^{\I'}=E^{\I}\setminus W$), and
  \item $E_1^{\I'}=E_1^{\I}$, for all symbols $E_1\neq E$.
  \end{compactenum}
\end{definition}

Now we can define the semantics of ground actions:

\begin{definition}
  Given a ground action $\alpha$, we define a mapping $S_{\alpha}$
  from interpretations to interpretations as follows:
  \[
    \small
    \renewcommand{\arraystretch}{1.1}
    \begin{array}[t]{rcl}
      S_{\varepsilon}(\I) &=& \I \\[1mm]
      S_{(A \oplus C )\cdot \alpha }(\I) &=& S_{\alpha}(\I\oplus_{A}C^{\I}) \\
      S_{(A \ominus C )\cdot \alpha }(\I) &=& S_{\alpha}(\I\ominus_{A}C^{\I})\\
      S_{(p \oplus r)\cdot \alpha }(\I) &=& S_{\alpha}(\I\oplus_{p}r^{\I})\\
      S_{(p\ominus r)\cdot \alpha }(\I) &=& S_{\alpha}(\I\ominus_{p}r^{\I})\\[1mm]
      S_{(\cact{\K}{\alpha_1}{\alpha_2})\cdot \alpha}(\I) & =&
      \begin{cases}
        S_{\alpha_1 \cdot \alpha}(\I), & \text{if } \I\models \K,\\
        S_{\alpha_2 \cdot \alpha}(\I), & \text{if } \I\not\models \K.
      \end{cases}
    \end{array}
  \]
\end{definition}

In the following, we assume that interpretations are updated using the above
language.

\begin{example}
  \label{example:inter-update}
  The following interpretation $\I_1$ represents (part of) the project database
  of some research institute. There are two active projects, and there are
  three employees that work in the active projects.
  \[
    \small
    \begin{array}{@{}r@{~}c@{~}l@{\qquad\qquad}r@{~}c@{~}l@{}}
      \msf{Prj}^{\I_1} &=& \{p_1,p_2\},
      & \msf{ActivePrj}^{\I_1} &=& \{p_1,p_2\},\\
      \msf{Empl}^{\I_1} &=& \{e_1,e_3,e_7\},
      & \msf{FinishedPrj}^{\I_1} &=& \{ \},\\
      \msf{worksFor}^{\I_1} &=&
      \multicolumn{4}{@{}l}{\{(e_1,p_1), (e_3,p_1), (e_7,p_2)\}.}
    \end{array}
  \]
  We assume constants $\msf{p_i}$ with $\msf{p_i}^{\I} = \mathit{p_i}$ for
  projects, and analogously constants $\msf{e_i}$ for employees.  The following
  action $\alpha_1$ captures the termination of project $p_1$, which is removed
  from the active projects and added to the finished ones.  The employees
  working only for this project are removed.
  \[
    \small
    \alpha_1 =
    \begin{array}[t]{@{}l}
      \msf{ActivePrj} \ominus \{\msf{p_1}\} \cdot
      \msf{FinishedPrj} \oplus \{\msf{p_1}\} \cdot{} \\
      \msf{Empl} \ominus \forall\msf{worksFor}.{\{\msf{p_1}\}}
    \end{array}
  \]
  The interpretation $S_{\alpha_1}(\I_1)$ that reflects the status of the
  database after action $\alpha_1$ looks as follows:
  \[
    \small
    \begin{array}{@{}r@{~}c@{~}l@{\qquad\qquad}r@{~}c@{~}l@{}}
      \msf{Prj}^{\I_1} &=& \{p_1,p_2\},
      & \msf{ActivePrj}^{\I_1} &=& \{p_2\},\\
      \msf{Empl}^{\I_1} &=& \{e_7\},
      & \msf{FinishedPrj}^{\I_1} &=& \{p_1\},\\
      \msf{worksFor}^{\I_1} &=&
      \multicolumn{4}{@{}l}{\{(e_1,p_1), (e_3,p_1), (e_7,p_2)\}.}
    \end{array}
  \]
\end{example}


Note that we have not defined the semantics of actions with variables, i.e.,
for non-ground actions.  In our approach, all variables of an action are seen
as parameters whose values are given before execution by a substitution with
actual individuals, i.e., by grounding.

\begin{example}\label{example:parameters}
  The following action
  $\alpha_2$ with variables $x$, $y$, $z$ transfers the employee $x$ from
  project $y$ to project $z$:
  \[
    \small
    \alpha_2 =
    \begin{array}[t]{l}
      \cactne{(x\,{:}\,\msf{Empl}\land y\,{:}\,\msf{Prj}\land z\,{:}\,\msf{Prj}
      \land (x,y)\,{:}\,\msf{worksFor})}{}\\
      (\msf{worksFor} \ominus \{(x,y)\} \cdot \msf{worksFor} \oplus \{(x,z)\})
    \end{array}
  \]
  Under the substitution $\sigma$ with
  $\sigma(x) = \msf{e_1}$,
  $\sigma(y) = \msf{p_1}$, and
  $\sigma(z) = \msf{p_2}$,
  the action $\alpha_2$ first checks whether $\msf{e_1}$ is an (instance of)
  employee, $\msf{p_1}$, $\msf{p_2}$ are projects, and $\msf{e_1}$ works for
  $\msf{p_1}$.  If yes, it removes the $\msf{worksFor}$ link between
  $\msf{e_1}$ and $\msf{p_1}$, and creates a $\msf{worksFor}$ link between
  $\msf{e_1}$ and $\msf{p_2}$.  If any of the checks fails, it does nothing.
\end{example}

\section{Capturing Action Effects}
\label{sec:capturing}

\newcommand{\STR}{\mathsf{TR}}

In this section we present our core technical tool: a transformation
$\STR_{\alpha}(\K)$ that rewrites $\K$ incorporating the possible effects of an
action $\alpha$. Intuitively, the models of $\STR_{\alpha}(\K)$ are exactly the
interpretations $\I$ such that applying $\alpha$ on $\I$ leads to a model of
$\K$.  In this way, we can effectively reduce reasoning about changes in any
database that satisfies a given $\K$, to reasoning about a single KB.  In the
next section we use this transformation to solve a wide range of data
management problems by reducing them to
standard DL reasoning services, such as finite (un)satisfiability.  This
transformation can be seen as a form of \emph{regression} \cite{LRLLS97}, which
incorporates the effects of a sequence of actions `backwards', from the last
one to the first one.

 \begin{definition}
  Given a KB $\K$, we use $\K_{E\gets E'}$ to denote the KB that is obtained
  from $\K$ by replacing every name $E$ by the (possibly more complex)
  expression $E'$. Given a KB $\K$ and an action $\alpha$, we define
  $\STR_{\alpha}(\K)$ as follows:
  \begin{align*}
   \STR_\varepsilon(\K)& = \K \\
   \STR_{(A \oplus C)\cdot
     \alpha}(\K)& = (\STR_{\alpha}(\K))_{A\gets A\OR C} \\
   \STR_{(A \ominus C)\cdot
     \alpha}(\K)& = (\STR_{\alpha}(\K))_{A\gets A\AND \neg C} \\
   \STR_{(p \oplus r)\cdot
     \alpha}(\K)& = (\STR_{\alpha}(\K))_{p\gets p\cup r} \\
   \STR_{(p \ominus r)\cdot
     \alpha}(\K)& = (\STR_{\alpha}(\K))_{p\gets p\setminus r} \\
   \STR_{(\cact{\K_1}{\alpha_1}{\alpha_2})\cdot
        \alpha}(\K) & =
      (\kbneg\K_1 \,{\lor}\,\STR_{\alpha_1\cdot\alpha}(\K)) \land {}
       (\K_1 \,{\lor}\,\STR_{\alpha_2\cdot\alpha}(\K)).
  \end{align*}
\end{definition}



Note that the size of $\STR_{\alpha}(\K)$ might be exponential in the size of
$\alpha$.  We now show that this transformation correctly captures the effects
of complex actions.

\begin{theorem}
  \label{thm:regression}
  Assume a ground action $\alpha$ and a KB $\K$. For every interpretation $\I$,
  we have $S_{\alpha}(\I)\models \K$ iff $\I\models \STR_{\alpha}(\K)$.
\end{theorem}

\begin{proof}
We define $s(\alpha)$ as follows: $s(\varepsilon) = 0$, $s(\beta \cdot \alpha)
= 1 + s(\alpha)$, and
$s(\cact{\K}{\alpha_1}{\alpha_2}\cdot\alpha_3) = 1 + s(\alpha_1) + s(\alpha_2)
+ s(\alpha_3)$.
We prove the claim by induction on 
$s(\alpha)$. In the base case
where $s(\alpha) = 0$ and $\alpha=\varepsilon$, we have $ S_{\alpha}(\I) = \I $ and
$\STR_{\alpha}(\K)=\K$ by definition, and thus the claim holds.

Assume $\alpha = (A \oplus C)\cdot \alpha'$. Let $\I'=\I\oplus_A\Int{\I}{C}$,
that is, $\I'$ coincides with $\I$ except that
$A^{\I'}=A^{\I}\cup C^{\I}$. 
For every KB $\K'$,
$\I'\models \K'$ iff $\I\models \K'_{A\gets A\OR C}$ (This can be proved by a
straightforward induction on the structure of the expressions in
$\K'$).  In particular,
$\I'\models \STR_{\alpha'}(\K)$ iff $\I\models
(\STR_{\alpha'}(\K))_{A\gets A\OR C}$.
Since $(\STR_{\alpha'}(\K))_{A\gets A\OR C}=\STR_{\alpha}(\K)$, we get
$\I'\models \STR_{\alpha'}(\K)$ iff $\I\models
\STR_{\alpha}(\K)$. By the induction hypothesis, $\I'\models
\STR_{\alpha'}(\K)$ iff $S_{\alpha'}(\I')\models \K$, thus
$\I\models \STR_{\alpha}(\K)$ iff $S_{\alpha'}(\I')\models
\K$. Since $S_{\alpha'}(\I')=S_{\alpha'}(S_{(A \oplus C)}(\I))=S_{\alpha}(\I)$
by definition,
we obtain $\I\models \STR_{\alpha}(\K)$ iff $S_{\alpha}(\I)\models \K$ as desired.

For the cases $\alpha = (A \ominus C)\cdot \alpha'$,
$\alpha = (p \oplus r)\cdot \alpha'$,
and $\alpha = (p \ominus r)\cdot \alpha'$, the argument is analogous.

Finally, we consider $\alpha=(\cact{\K_1}{\alpha_1}{\alpha_2}) \cdot \alpha'$, and assume an
arbitrary $\I$.  We consider the case where $\I\models \K_1$; the case  where
$\I\not\models \K_1$ is analogous. By definition
$S_{\alpha}(\curlyI)=S_{\alpha_1 \cdot \alpha'}(\curlyI)$.
By the induction hypothesis we
  know that $S_{\alpha_1 \cdot \alpha'}(\I) \models \K$ iff $\I \models
  \STR_{\alpha_1 \cdot \alpha'}(\K)$, so $S_{\alpha}(\curlyI)  \models \K$ iff
  $\I \models \STR_{\alpha_1 \cdot \alpha'}(\K)$.
Since $\I \models \K_1$ and
$\STR_{(\cact{\K_1}{\alpha_1}{\alpha_2})\cdot
        \alpha}(\K) =
         (\kbneg\K_1 \,{\lor}\,\STR_{\alpha_1\cdot\alpha}(\K)) \land (\K_1
         \,{\lor}\,\STR_{\alpha_2\cdot\alpha}(\K))$, it follows that
$S_{\alpha}(\curlyI)  \models \K$ iff
  $\I \models \STR_{(\cact{\K_1}{\alpha_1}{\alpha_2})\cdot
        \alpha}(\K)$.
\end{proof}

This theorem will be important for solving the reasoning problems we study
below.

\begin{example}
  The following KB $\K_1$ expresses constraints on the project database of our
  running example: all projects are active or finished, the domain of
  $\msf{worksFor}$ are the employees, and its range the projects.
  \[
    \small
    \begin{array}{l}
      (\msf{Prj} \ISA \msf{ActivePrj} \OR \msf{FinishedPrj}) \land{}\\
      (\exists\msf{worksFor}.\top \ISA \msf{Empl})  \land{}\\
      (\exists\msf{worksFor}^-.\top \ISA \msf{Prj})
    \end{array}
  \]
  By applying the transformation above to $\K_1$ and $\alpha_1$, we obtain the
  following KB $\STR_{\alpha_1}(\K_1)$:
  \[
    \small
    \begin{array}{l}
      (\msf{Prj} \ISA
      (\msf{ActivePrj} \AND \NOT\{\msf{p_1}\}) \OR
      (\msf{FinishedPrj} \OR \{\msf{p_1}\})) \land{}\\
      (\exists\msf{worksFor}.\top \ISA
      \msf{Empl} \AND \exists\msf{worksFor}.{\NOT\{\msf{p_1}\}}) \land{}\\
      (\exists\msf{worksFor}^-.\top \ISA \msf{Prj})
    \end{array}
  \]
\end{example}


\section{Static Verification}

In this section, we consider the scenario where DL KBs are used to impose
integrity constraints on GSD.  One of the most basic reasoning problems for
action analysis in this setting is \emph{static verification}, which consists
in checking whether the execution of an action $\alpha$ always preserves the
satisfaction of integrity constraints given by a KB.


\begin{definition}[The static verification problem]
  Let $\K$ be a KB. We say that an action $\alpha$ is \emph{$\K$-preserving} if
  for every ground instance $\alpha'$ of $\alpha$ and every finite
  interpretation $\curlyI$, we have that $\curlyI\models \K$ implies
  $S_{\alpha'}(\curlyI)\models \K$.  The \emph{static verification problem} is
  defined as follows:
\begin{enumerate}[(SV)]
\item Given  an action $\alpha$ and a KB $\K$, is $\alpha$ $\K$-preserving?
\end{enumerate}
\end{definition}



Using the transformation $\STR_{\alpha}(\K)$ above, we {can} reduce
static verification to finite (un)satisfiability of \OURDL KBs:  An action
$\alpha$ is not $\K$-preserving iff some finite model of $\K$ does not satisfy
$\STR_{\alpha^{*}}(\K)$, where $\alpha^{*}$ is a `canonical' grounding of
$\alpha$.
Formally, we have:

\begin{theorem}\label{thm:verification}
  Assume a (complex) action $\alpha$ and a KB $\K$. Then the following
  are equivalent:
  \begin{asparaenum}[(i)]
  \item The action $\alpha$ is not $\K$-preserving.
  \item $\K\land \kbneg \STR_{\alpha^{*}}(\K)$ is finitely satisfiable, where
    $\alpha^{*}$ is obtained from $\alpha$ by replacing each variable with a
    fresh individual name not occurring in $\alpha$ and~$\K$.
  \end{asparaenum}
\end{theorem}

\begin{example}
  \label{example:verification}
  The action $\alpha_1$ from Example~\ref{example:inter-update} is not
  $\K_1$-preserving:
  $\I_1 \models \K_1$, but $S_{\alpha_1}(\I_1) \not\models\K_1$ since the
  concept inclusion {\small $\exists\msf{worksFor}.\msf{Prj} \ISA \msf{Empl}$}
  is violated.  This is reflected in the fact that
  $\I_1\not\models\STR_{\alpha_1}(\K_1)$, as can be readily checked.
  Intuitively, values removed from $\msf{Empl}$ should also be removed from
  $\msf{worksFor}$, as in the following $\K_1$-preserving action:
  \[
    \small
    \alpha'_1 =
    \begin{array}[t]{@{}l}
      \msf{ActivePrj} \ominus \{\msf{p_1}\} \cdot
      \msf{FinishedPrj} \oplus  \{\msf{p_1}\} \cdot {} \\
      \msf{Empl} \ominus \forall\msf{worksFor}.{\{\msf{p_1}\}} \cdot
      \msf{worksFor} \ominus \msf{worksFor}|_{\{\msf{p_1}\}}
    \end{array}
  \]
\end{example}

The above theorem provides an algorithm for static verification, which we can
also use to obtain tight bounds on the computational complexity of the problem.
Indeed, even though $\K\land \kbneg\STR_{\alpha^{*}}(\K)$ may be of size
exponential in $\alpha$, we can avoid to generate it all at once.  More
precisely, we use a non-deterministic polynomial time many-one reduction that
builds only $\K\land \kbneg\STR^c_{\alpha^{*}}(\K)$ for a fragment
$\kbneg\STR^c_{\alpha}(\K)$ of $\kbneg \STR_{\alpha^{*}}(\K)$ that corresponds
to one fixed way of choosing one of $\alpha_1$ or $\alpha_2$ for each
conditional action $\cact{\K'}{\alpha_1}{\alpha_2}$ in $\alpha$ (intuitively,
we can view $\kbneg \STR^c_{\alpha^{*}}(\K)$ as one conjunct of the DNF of
$\kbneg \STR_{\alpha}(\K)$, where axioms and assertions are treated as
propositions). Such a $\kbneg \STR^c_{\alpha}(\K)$ has polynomial size, and it
can be built non-deterministically in polynomial time.  It is not hard to show
that $\K\land \kbneg \STR_{\alpha^{*}}(\K)$ is finitely satisfiable iff there
is some choice $\STR^c_{\alpha^{*}}(\K)$ such that $\K\land \kbneg
\STR^c_{\alpha^{*}}(\K)$ is finitely satisfiable.
  By Theorem~\ref{thm:dl-sat}, the latter test
  can be done in non-deterministic exponential time, hence from
  Theorem~\ref{thm:verification} we obtain:


\begin{theorem}\label{thm:verification-nexp}
  The problem (SV) is co\nexptime-complete in case the input KB is expressed in
  \OURDL.
\end{theorem}

We note that in our definition of the (SV) problem, in addition to the action
to be verified, one has as input only one KB $\K$ expressing constraints.  We
can also consider other interesting variations of the problem where, for
example, we have a pair of KBs $\K_\mathit{pre}$ and $\K_\mathit{post}$ instead
of (or in addition to) $\K$ and we want to decide whether executing the action
on any model of $\K_{\mathit{pre}}$ (and $\K$) leads to a model of
$\K_\mathit{post}$ (and $\K$).  The reasoning techniques and upper bounds
presented above also apply to these generalized settings.


\subsection{Lowering the Complexity}

The goal of this section is to identify a setting for which the computational
complexity of static verification is lower.  The natural way to achieve this is
to consider as constraint language a
DL with better computational properties, such as the logics of the \dlc
family \cite{CDLLR07}.

Unfortunately, we cannot achieve tractability, since static verification is
co\np hard even in a very restricted setting, as shown next.

\begin{theorem}
  \label{thm:conp-hardness}
  The static verification problem is co\np-hard already for KBs of the form
  $(A_1\ISA \neg A_1')\land \cdots \land (A_n\ISA \neg A_n') $, where each
  $A_i,A_i' $ is a concept name, and
  ground sequences of basic actions of the forms $(A \oplus C)$ and $(A \ominus
  C)$.
\end{theorem}

We next present a rich variant of \dlh,
which we call \dlhp, for which the static verification problem is in co\np.
It supports (restricted) Boolean combinations of inclusions and assertions, and
allows for complex concepts and roles in assertions.  As shown below,
this allows us to express the effects of actions inside \dlhp KBs.

\begin{definition}
  The logic \emph{\dlhp} is defined as follows:
  \begin{enumerate}[-]
   \item Concept inclusions have the form $C_1\ISA C_2$ or
    $C_1\ISA \neg C_2$, with $C_1, C_2\in \conceptnames\cup
    \{\SOME{p}{\top}, \SOME{p^-}{\top} \mid p\in \rolenames \}.$
  \item Role inclusions in $\K$ have the form $r_1\ISA r_2$ or $r_1\ISA
    \NOT r_2$, with $r_1,r_2\in \rolenames\cup\{p^-\mid p\in\rolenames \}$.
  \item Role assertions are defined as for \OURDL, but in concept assertions
    $o: C$, we require $C\in \boldB^{+}$, where $\boldB^{+}$ is the smallest
    set of concepts such that:
    \begin{compactenum}[(a)]
    \item $\conceptnames\subseteq \boldB^{+}$,
    \item $\{o'\}\in \boldB^{+}$ for all $o'\in \indivnames$,
    \item $\exists r.\top\in \boldB^+$ for all roles $r$,
    \item $\{B_1 \AND B_2,B_1 \OR B_2, \neg B_1\}\subseteq \boldB^+$ for all
      $B_1, B_2\in\boldB^+$.
    \end{compactenum}
  \item Formulae and KBs are defined as for \OURDL, but the operator $\kbneg$ may occur only in front of assertions.
  \end{enumerate}
  A \dlh KB $\K$ is a \dlhp KB that satisfies the following restrictions:
  \begin{compactenum}[-]
  \item \K is a \emph{conjunction} of inclusions and assertions, and
  \item all assertions in $\K$ are \emph{basic assertions} of the forms $o:A$
    with $A\in \conceptnames$, and $(o,o'):p$ with $p\in\rolenames$.
  \end{compactenum}
  We make the \emph{unique name assumption (UNA)}: for every pair of
  individuals $o_1$, $o_2$ and interpretation $\I$, we have $o_1^{\I}\neq
  o_2^\I$.
\end{definition}

We need to slightly restrict the action language, which involves allowing only
Boolean combinations of
assertions to express the condition $\K$ in actions of the form
$\cact{\K}{\alpha_1}{\alpha_2}$.
\begin{definition}
  A (complex) action $\alpha$ is called \emph{simple} if
  \begin{inparaenum}[(i)]
  \item no (concept or role) inclusions occur in $\alpha$, and
  \item all concepts of $\alpha$ are from $\boldB^+$.
  \end{inparaenum}
\end{definition}

We next characterize the complexity of finite satisfiability in
  \dlhp.

\begin{theorem}
  \label{thm:dllite-complexity}
  Finite satisfiability of \dlhp KBs is \np-complete.
\end{theorem}

\noindent
\dlhp is expressive enough to allow us to reduce static verification for simple
actions to finite unsatisfiability, 
{ and similarly as above, we
 can use a non-deterministic polynomial time many-one reduction (from the
 complement of static verification to finite unsatisfiability)
 to obtain a co\np upper bound  on the complexity of static
  verification.
This bound is tight, even if we allow only actions with preconditions
rather than full conditional actions. We note that all lower bounds in
the next section also hold for this restricted case.}

\begin{theorem}
\label{thm:reduction-dllite}
The static verification problem for \dlhp KBs and simple actions is
  co\np-complete.
\end{theorem}


\section{Planning}
\label{sec:planning}

We have focused so far on ensuring that the satisfaction of constraints is
preserved when we evolve GSD.  But additionally,
there may be desirable states of the GSD that we want to achieve, or
undesirable ones that we want to avoid.  For instance, one may want to ensure
that a finished project is never made active again.
This raises several problems, such as deciding if there exists a sequence of
actions to reach a state with certain properties, or whether a given sequence
of actions always ensures that a state with certain properties is reached.
We consider now these problems and formalize them by means of \emph{automated
 planning}.

We use DLs
to describe states of KBs, which may act as goals or preconditions.  A
\emph{plan} is a sequence of actions \emph{from a given set}, whose execution
leads an agent from the current
state to a state that satisfies a given goal.
\begin{definition}\label{def:planning}
  Let $\I=\tuple{\Delta^{\I},\cdot^{\I}}$ be a finite interpretation, $\Act$ a {finite}
  set of actions, and $\K$ a KB (the \emph{goal} KB).  A {finite} sequence
  $\tuple{\alpha_1,\ldots,\alpha_n}$ of ground instances of actions from $\Act$
  is called a \emph{plan for $\K$ from $\I$} (of \emph{length} $n$), if there
  exists a finite set $\Delta$ with $\Delta^{\I}\cap \Delta= \emptyset$ such
  that $S_{\alpha_1\cdots\alpha_n}(\I')\models \K$, where
  $\I'=\tuple{\Delta^{\I}\cup \Delta,\cdot^{\I}}$.
\end{definition}

Recall that actions in our setting do not modify the domain of an
interpretation. To support unbounded introduction of values in the data, the
definition of planning above allows for the domain to be expanded a-priori with
a finite set of fresh domain elements.

We can now define the first planning problems we study:
\begin{enumerate}[(P1)]
\item Given a set $Act$ of actions, a finite interpretation $\I$, and a goal KB
  $\K$, does there exist a plan for $\K$ from $\I$?
\item Given a set $Act$ of actions and a pair $\K_\mathit{pre}$, $\K$ of
  formulae, does there exist a substitution $\sigma$ and a plan for
  $\sigma(\K)$ from some finite $\I$ with $\I\models \sigma(\K_\mathit{pre})$?
\end{enumerate}
(P1) is the classic plan existence problem, formulated in the setting of GSD.
(P2)
also aims at deciding plan existence, but rather than the full actual state of
the data, we have as an input a \emph{precondition} KB, and we are interested
in deciding the existence of a plan from some of its models.
To see the relevance of (P2), consider the complementary problem: a `no'
instance of (P2) means that, from every relevant initial state, (undesired)
goals cannot be reached.
For instance, $\K_\mathit{pre}=\K_{ic}\land x:\msf{FinishedPrj}$ and $\K=
x:\msf{ActivePrj}$ may be used to check whether starting with GSD that
satisfies the integrity constraints and contains some finished project $p$, it
is possible to make $p$ an active project again.

\begin{example}
  Recall the interpretation $\I_1$ and the action $\alpha'_1$ from
  Example~\ref{example:verification}, and  the substitution $\sigma$
  from Example~\ref{example:parameters}, which gives us the following ground
  instance 
  of $\alpha_2$:
  \[
    \small
    \alpha_2'=
    \begin{array}[t]{@{}l@{}}
      \cactne{(\msf{e_1}:\msf{Empl} \land \msf{p_1}:\msf{Prj} \land
       \msf{p_2}:\msf{Prj} \land (\msf{e_1},\msf{p_1}):\msf{worksFor})}{}\\
      {(\msf{worksFor} \ominus \{(\msf{\msf{e_1}},\msf{p_1})\} \cdot
       \msf{worksFor} \oplus \{(\msf{\msf{e_1}},\msf{\msf{p_2}})\})}
    \end{array}
  \]

  The following \emph{goal} KB requires that $\msf{p_1}$ is not an active
  project, and that $\msf{e_1}$ is an employee. 
  \[
    \small
    \K_g = \kbneg(\msf{p_1}\,{:}\,\msf{ActivePrj}) \land
    \msf{e_1}\,{:}\,\msf{Empl}
  \]
  A plan for $\K_g$ from $\I_1$ is the sequence of actions
  $\tuple{\alpha'_2,\alpha'_1}$.
  The interpretation $S_{\alpha'_2\cdot\alpha'_1}(\I_1)$ that reflects the status
  of the data after applying
  $\tuple{\alpha'_2,\alpha'_1}$ looks as follows:
  \[
    \small
    \begin{array}{r@{~}c@{~}l}
      \msf{Prj}^{S_{\alpha'_2\cdot\alpha'_1}(\I_1)} &=& \{p_1,p_2\}\\
      \msf{ActivePrj}^{S_{\alpha'_2\cdot\alpha'_1}(\I_1)} &=& \{p_2\}\\
      \msf{Empl}^{S_{\alpha'_2\cdot\alpha'_1}(\I_1)} &=& \{e_1,e_7\}\\
      \msf{FinishedPrj}^{S_{\alpha'_2\cdot\alpha'_1}(\I_1)} &=& \{p_1\}\\
      \msf{worksFor}^{S_{\alpha'_2\cdot\alpha'_1}(\I_1)} &=& \{(e_1,p_2),(e_7,p_2)\}
    \end{array}
  \]
  Clearly,  $S_{\alpha'_2\cdot\alpha'_1}(\I_1)\models \K_1$.
\end{example}

Unfortunately, these problems are undecidable in general, which can
 be shown by a reduction from the Halting problem for Turing machines.

\begin{theorem}
  \label{thm:plan-existence-undecidable}
  The problems (P1) and (P2) are undecidable, already for \dlhp KBs and simple
  actions.
\end{theorem}
{Intuitively,  problem~(P1) is undecidable because we cannot
  know  
how many fresh objects need to be added to the domain of $\I$, but it becomes
decidable 
if
the size of $\Delta$ in
Definition~\ref{def:planning} is bounded.}
It is not difficult to
see that problem~(P2) remains undecidable even if the domain is assumed fixed
(as the problem definition quantifies existentially over interpretations, one
can choose interpretations with sufficiently large domains).  However, also
(P2) becomes decidable if we place a bound on the length of plans. More
precisely, the following problems are decidable.


\begin{enumerate}[(P2\textsubscript{b})]
\item[(P1\textsubscript{b})] 
{Given a set $\mathit{Act}$ of actions, a finite
  interpretation $\I$, a goal KB $\K$,  and a positive integer $k$, does there exist a plan
 for $\K$ from $\I$ where $|\Delta| \leq k$?}

\item[(P2\textsubscript{b})] Given a set of actions $Act$, a pair
  $\K_\mathit{pre},\K$ of formulae, and a positive integer $k$, does there exist a substitution $\sigma$ and a plan of length
  $\leq k$ for $\sigma(\K)$ from some finite interpretation $\I$ with
  $\I\models \sigma(\K_\mathit{pre})$?
\end{enumerate}

{We now study the complexity of these problems, assuming that the
  input bounds $k$ are coded in unary.}
The problem (P1\textsubscript{b})  can be solved in polynomial space, and
thus is not harder than deciding the existence of a plan in standard automated
planning formalisms such as propositional STRIPS~\cite{Bylander94strips}.
{In fact, the following lower bound can be proved by a reduction from
 the latter formalism, or by an adaptation of the Turing Machine reduction used
 to prove undecidability in Theorem~\ref{thm:plan-existence-undecidable}.}

\begin{theorem}
  \label{thm:finite-plan-existence-p1-pspace}
  The problem (P1\textsubscript{b}) is \pspace-complete for
  \OURDL KBs.
\end{theorem}


Now we establish the complexity of (P2\textsubscript{b}), both in the general
setting (i.e.,\,when $\K_\mathit{pre}$ and $\K$ are in \OURDL), and for the
restricted case of \dlhp KBs and simple actions.  For (SV), considering the
latter setting allowed us to reduce the complexity from co\nexptime to
co\np. Here we obtain an analogous result and go from \nexptime-completeness to
\np-completeness.

\begin{theorem}
  \label{thm:finite-plan-existence-p2-nexptime}
  The problem (P2\textsubscript{b}) is \nexptime-complete.  It is \np-complete
  if $\K_\mathit{pre},\K$ are expressed in \dlhp and all actions in $\Act$ are
  simple.
\end{theorem}


Now we consider three problems that are related to ensuring plans that
\emph{always} achieve a given goal, no matter what the initial data is. They
are variants of the so-called \emph{conformant} planning, which deals with
planning under various forms of incomplete information.  In our case, we assume
that we have an incomplete description of the initial state, since we only know
it satisfies a given precondition, but have no concrete interpretation.

The first of such problems is to `certify' that a candidate plan is indeed a
plan for the goal, for every possible database satisfying the precondition.


\begin{enumerate}[(C)]
\item Given a sequence $P=\tuple{\alpha_1,\ldots,\alpha_n}$ of actions and
  formulae $\K_\mathit{pre}$, $\K$, is $\sigma(P)$ a plan for $\sigma(\K)$ from
  every finite interpretation $\I$ with $\I\models \sigma(\K_\mathit{pre})$,
  for every possible substitution $\sigma$?
\end{enumerate}

Finally, we are interested in the existence of a plan that always achieves the
goal, for every possible state satisfying the precondition.  Solving this
problem corresponds to the automated \emph{synthesis} of a program for reaching
a certain condition.  We formulate the problem with and without a bound on the
length of the plans we are looking for.

\begin{enumerate}[(S)]
\item[(S)] Given a set $Act$ of actions and formulae $\K_\mathit{pre}$, $\K$,
  does there exist a sequence $P$ of actions such that $\sigma(P)$ is a plan
  for $\sigma(\K)$ from every finite interpretation $\I$ with
  $\I\models\sigma(\K_\mathit{pre}$), for every possible substitution $\sigma$?

\item[(S\textsubscript{b})] Given a set $Act$ of actions, formulae
  $\K_\mathit{pre},\K$, and a positive integer $k$, does there exist a sequence
  $P$ of actions such that $\sigma(P)$ is of length at most $k$ and is a plan
  for $\sigma(\K)$ from every finite interpretation $\I$ with $\I\models
  \sigma(\K_\mathit{pre})$, for every possible substitution $\sigma$?
\end{enumerate}

We conclude  with the complexity of these problems:
\begin{theorem}
  \label{thm:planning-complexity}
  The following hold:
  \begin{enumerate}[-]
  \item Problem (S) is undecidable, already for \dlhp KBs and simple actions.
  \item Problems (C) and (S\textsubscript{b}) are co\nexptime-complete.
  \item If $\K_\mathit{pre},\K$ are expressed in \dlhp and all actions in
    $\Act$ are simple, then (C) is co\np-complete and (S\textsubscript{b}) is
    $\np^{\np}$-complete.
  \end{enumerate}
\end{theorem}

\section{Related Work}
\label{sec:related}

Using DLs to understand the properties of systems while fully taking into
account both structural and dynamic aspects is very challenging \cite{WoZa99}.
Reasoning in DLs extended with a temporal dimension becomes quickly undecidable~\cite{Arta06}, unless severe restrictions on the expressive power of the DL are
imposed \cite{AKRZ11}.  An alternative approach to achieve decidability is to
take a so-called ``functional view of KBs'' \cite{Leve84}, according to which
each state of the KB can be queried via logical implication, and the KB is
progressed from one state to the next through forms of update \cite{CDLR11}.
This makes it possible (under suitable conditions) to \emph{statically verify}
(temporal) integrity constraints over the evolution of a system
\cite{BaGL12,BCMD*13}.

Updating databases, and logic theories in general, is a classic topic in
knowledge representation, discussed extensively in the literature, cf.\
\cite{FKUV86,KatsunoM91}.  The updates described by our action
language are similar in spirit to the knowledge base updates studied in other
works, and in particular, the ABox updates considered in \cite{LLMW11},
and \cite{KhZC13}. As our updates are done directly on interpretations
rather than on (the instance level of) knowledge
bases, 
we do not encounter the expressibility and succinctness problems faced there.

Concerning the reasoning problems we tackle, verifying consistency of
transactions is a crucial problem that has been studied extensively in
Databases. It has been considered for different kinds of transactions and
constraints, over traditional relational databases \cite{ShSt89},
object-oriented databases \cite{SpBa98,BoKi94},
and deductive databases \cite{KoSS87}, to name a few. Most of these works adopt
expressive formalisms like (extensions of) first or higher order predicate
logic \cite{BoKi94}, or undecidable tailored languages \cite{ShSt89} to express
the constraints and the operations on the data. Verification systems are often
implemented using theorem provers, and complete algorithms cannot be devised.

As
mentioned, the problems studied in Section~\ref{sec:planning} are closely
related to automated planning, a topic extensively studied in AI.
DLs have been employed to reason about actions, goals, and plans, as well as
about the application domains in which planning is deployed, see \cite{Gil05}
and its references.  Most relevant to us is the significant body of work on
DL-based action languages \cite{BLMSW05,Milicic2008,BaLL10,LLMW11,BaZa13}.
In these formalisms, DL constructs are used to give conditions on the effects
of action execution, which are often non-deterministic.  A central problem
considered is the \emph{projection problem}, which consists in deciding whether
every possible execution of an action sequence on a possibly incomplete state
will lead to a state that satisfies a given property.  Clearly, our
certification problem (C), which involves an incomplete
initial state, is a variation of the projection problem.  However, we do not
face the challenge of having to consider different possible executions of
non-deterministic actions.  Many of our other reasoning problems are similar to
problems considered in these works, in different forms and contexts.  A crucial
difference is that our well-behaved action language allows us to obtain
decidability even when we employ full-fledged TBoxes for specifying goals,
preconditions, and domain constraints.  To the best of our knowledge, previous
results rely on TBox acyclicity to ensure decidability.

\section{Conclusions}
\label{sec:conclusions}

We have considered graph structured data that evolve as a result of updates
expressed in a powerful yet well-behaved action language.  We have studied
several reasoning problems that support the static analysis of actions and
their effects on the state of the data.  We have shown the decidability of most
problems, and in the cases where the general problem is undecidable, we have
identified decidable restrictions
and have characterized the computational complexity
for a very expressive DL and a variant of \dlc.  We believe this work provides
powerful tools for analyzing the effects of executing complex actions on
databases, possibly in the presence of integrity constraints expressed in rich
DLs.  Our upper bounds rely on a novel KB transformation technique,
which enables to reduce most of the reasoning tasks to
finite (un)satisfiability in a DL.
This calls for
developing finite model reasoners for DLs (we note that \OURDL does not have
the finite model property). It also remains to better understand the complexity
of finite model reasoning in different variations of \dlc. E.g., extensions of
\dlhp with role functionality would be very useful in the context of graph
structured data.  Generalizing the positive decidability results to logics
with powerful identification constraints, like the ones considered
in~\cite{rnf2013}, would also be of practical importance.  Given that the considered problems are intractable even for
weak fragments of the core \dlc and very restricted forms of actions, it
remains to explore how feasible these tasks are
in practice, and whether there are meaningful restrictions that make them
tractable.


\section*{Acknowledgments}
This research has been partially supported by FWF
projects~T515-N23 and P25518-N23, by WWTF project~ICT12-015, by EU IP
Project Optique~FP7-318338, and by the Wolfgang Pauli Institute.

\clearpage
\bibliographystyle{plain}
\bibliography{main-bib}

 \clearpage
 \section*{Appendix}

\begin{proof}[Proof of Theorem~\ref{thm:verification}]
  \textit{(i)} to \textit{(ii)}.
  Assume there exist a ground instance $\alpha'$ of $\alpha$ and
  a finite interpretation $\curlyI$ such that $\curlyI\models \K$ and
  $S_{\alpha'}(\curlyI)\not\models \K$. Then by Theorem~\ref{thm:regression}, $
  \I\not\models \STR_{\alpha'}(\K)$. Thus $ \I\models \neg
  \STR_{\alpha'}(\K)$. Suppose $o_1\rightarrow x_1,\ldots,o_n\rightarrow x_n$
  is the substitution that transforms $\alpha$ into $\alpha'$. Suppose also
  $o_1'\rightarrow x_1,\ldots,o_n'\rightarrow x_n$ is the substitution that
  transforms $\alpha$ into $\alpha^*$. Take the interpretation $\I^*$ that
  coincides with $\I$ except for $(o_i')^{\I^*}=(o_i)^{\I}$. Then $\I^*\models
  \K\land \neg \STR_{\alpha^{*}}(\K)$.

  \textit{(ii)} to \textit{(i)}.
  Assume $\K\land \neg \STR_{\alpha^{*}}(\K)$ is finitely
  satisfiable, i.e., there is an interpretation $\I$ such that $\I\models \K$
  and $\I\not \models \STR_{\alpha^{*}}(\K)$. Then by
  Theorem~\ref{thm:regression}, $S_{\alpha^*}(\curlyI)\not\models \K$.
\end{proof}

\newcommand{\tro}{\overline{\mathsf{TR}}}
\newcommand{\tron}{\tro'}
\newcommand{\TB}{\overline{\mathbf{TR}}}
\newcommand{\TBC}{\mathbf{TR}^{\wedge}}

\begin{proof}[Proof of Theorem~\ref{thm:verification-nexp}]
For co\nexptime-hardness, we note that finite unsatisfiability of \OURDL KBs can be reduced
in polynomial time to static verification in the presence of \OURDL
KBs. Indeed, a KB $\K$ is finitely satisfiable iff $(A' \oplus \{o\})$ is not
$(\K\land (A \ISA \NOT A') \land (o:A))$-preserving, where $A$, $A'$ are fresh
concept names and $o$ is a fresh individual.

\added[MO]
{Obtaining a matching upper bound is slightly more involved. 
It follows from Theorem~\ref{thm:verification} that the complement of
static verification in the presence of \OURDL
KBs reduces
to finite satisfiability of a KB $\K \land \kbneg \STR_{\alpha^*}(\K)$ in
\OURDL, but unfortunately, this reduction is exponential
in general. Hence
we use an alternative 
reduction that allow us to \emph{non-deterministically}   build in polynomial
time a formula $\K'$ of polynomial size, such that   $\K \land \K'$ is
satisfiable iff $\K \land \kbneg \STR_{\alpha^*}(\K)$ is satisfiable. The upper
bound then follows from this and the fact that finite satisfiability in \OURDL
is \nexptime-complete (c.f.\ Theorem~\ref{thm:dl-sat}).}
\todo{the rest of the proof is new}

To obtain this non-deterministic   polynomial time many-one reduction, it is
convenient to first define a minor variation $\tro_\alpha(\K)$ of the
transformation above, which generates an already negated KB.
\begin{align*}
\tro_\varepsilon(\K)= & \kbneg \K \\
 \tro_{(A \oplus C)\cdot
     \alpha}(\K)= &(\tro_{\alpha}(\K))_{A\gets A\OR C} \\
  \tro_{(A \ominus C)\cdot
     \alpha}(\K)= &(\tro_{\alpha}(\K))_{A\gets A\AND \neg C} \\
   \tro_{(p \oplus r)\cdot
     \alpha}(\K)= &(\tro_{\alpha}(\K))_{p\gets p\cup r} \\
  \tro_{(p \ominus r)\cdot
     \alpha}(\K)= &(\tro_{\alpha}(\K))_{p\gets p\setminus r} \\
\tro_{(\cact{\K_1}{\alpha_1}{\alpha_2})\cdot
        \alpha}(\K) = &
        \big (\K_1 \,{\land}\,\tro_{\alpha_1\cdot\alpha}(\K) \big ) \lor
      \big ( \kbneg \K_1 \,{\land}\,\tro_{\alpha_2\cdot\alpha}(\K) \big )
  \end{align*}
It can be shown by a straightforward induction on $s(\alpha)$  (as defined in
the Proof of Theorem~\ref{thm:regression}) that
$\tro_\alpha(\K)$ is logically equivalent to $\kbneg \STR_\alpha(\K)$ for
every $\K$ and every $\alpha$.
Hence, by Theorem \ref{thm:regression},
$\K \land \tro_{\alpha^*}(\K)$ is finitely satisfiable iff
$\K \land \kbneg \STR_{\alpha^*}(\K)$ is finitely satisfiable iff
 $\alpha$ is not $\K$-preserving.

Now, for the desired reduction, we use a non-deterministic version
of $\tro_\alpha(\K)$ that is defined analogously 
but in the last case, for the conditional axioms, we non-deterministically
choose between $\K_1 \,{\land}\,\tro_{\alpha_1\cdot\alpha}(\K)$, or
$ \kbneg \K_1 \,{\land}\,\tro_{\alpha_2\cdot\alpha}(\K)$, rather than
considering the disjunction of both. We denote by $\TB_\alpha(\K)$ the set of all the KBs  obtained this way, that is:
\begin{align*}
  \TB_\varepsilon(\K)= &  \{ \kbneg \K \} \\
   \TB_{(A \oplus C)\cdot
     \alpha}(\K)= &  \{ \K'_{A\gets A\OR C} \mid \K' \in \TB_{\alpha}(\K)\} \\
   \TB_{(A \ominus C)\cdot
     \alpha}(\K)= &  \{ \K'_{A\gets A\AND \neg C} \mid \K' \in \TB_{\alpha}(\K)\} \\
   \TB_{(p \oplus r)\cdot
     \alpha}(\K)= &  \{ \K'_{p\gets p\cup r} \mid \K' \in
   \TB_{\alpha}(\K) \}  \\
   \TB_{(p \ominus r)\cdot
     \alpha}(\K)= & \{ \K'_{p\gets p\setminus r} \mid \K' \in \TB_{\alpha}(\K)\} \\
   \TB_{(\cact{\K_1}{\alpha_1}{\alpha_2})\cdot
        \alpha}(\K) = &  \{ \K_1 \,{\land}\, \K' \mid \K' \in
      \TB_{\alpha_1\cdot\alpha}(\K)  \} \cup \{ \kbneg \K_1 \,{\land}\, \K'
      \mid \K' \in \TB_{\alpha_2\cdot\alpha}(\K) \}
\end{align*}
It is easy to see that $|\TB_\alpha(\K)|$ may be
exponential in $\alpha$ and $\K$, but each $\K' \in \TB_\alpha(\K)$ is
of polynomial size and can be built (non-deterministically) in polynomial
time.  It is only left to show that $\K \land \tro_{\alpha}(\K)$ is finitely satisfiable iff
there is some $\K' \in \TB_\alpha(\K)$ such that $\K \land \K'$ is finitely
satisfiable.
This is a consequence of the fact that, for every interpretation $\I$,
 $\I \models \tro_{\alpha}(\K)$ iff there is some $\K' \in \TB_\alpha(\K)$
such that $\I \models \K'$.

 We show this by induction on
$s(\alpha)$.
The base case is straightforward: if $\alpha = \epsilon$,
then $\TB_\alpha(\K) = \{\tro_{\alpha}(\K)\}$.
For the inductive step, we first consider $\alpha  = (A \oplus C) \cdot
\alpha'$.
First we assume that
$\I \models \tro_{\alpha}(\K)$. That is,
$\I \models (\tro_{\alpha'}(\K))_{A\gets A\OR C}$. We can apply the induction
hypothesis to infer that there exists $\K' \in \TB_{\alpha'}(\K)$ such that
$\I \models \K'_{A\gets A\OR C}$, which implies that exists $\K'' = \K'_{A\gets A\OR C}$
such that $\K'' \in \TB_{\alpha}(\K)$ and $\I \models \K''$ as desired.
For the converse, if  $\I \models \K''$ for some  $\K'' \in
\TB_{\alpha}(\K)$, by definition
we have that there is some $\K' \in \TB_{\alpha'}(\K)$ such that $\I \models
\K'_{A\gets A\OR C}$. Using the induction hypothesis we get  $\I \models
\TB_{\alpha'}(\K)_{A\gets A\OR C}$, that is,  $\I \models\TB_{\alpha}(\K)$
as desired.
The cases of  $\alpha  = (A \ominus C) \cdot \alpha'$,   $\alpha  = (p \oplus
r) \cdot \alpha'$,  and   $\alpha  = (p \ominus r) \cdot \alpha'$ are
analogous.

Finally, consider $\alpha = (\cact{\K_1}{\alpha_1}{\alpha_2})\cdot
        \alpha'$. We first show that if $\I \models \tro_{\alpha}(\K)$, then
there is some $\K' \in \TB_\alpha(\K)$ such that $\I \models \K'$.
By definition,  $\tro_{\alpha}(\K) = \big (\K_1 \,{\land}\,\tro_{\alpha_1\cdot\alpha}(\K) \big ) \lor
      \big ( \kbneg \K_1 \,{\land}\,\tro_{\alpha_2\cdot\alpha}(\K) \big )$.
  So, if $\I \models \tro_{\alpha}(\K)$, then one of $\I \models \K_1
  \,{\land}\,\tro_{\alpha_1\cdot\alpha}(\K)$ or
  $\I \models \kbneg \K_1 \,{\land}\,\tro_{\alpha_2\cdot\alpha}(\K)$ holds.
  In the former case, we can use the induction hypothesis to
  conclude that there exists some $\K' \in \TB_{\alpha_2\cdot\alpha}(\K)$ such
  that $\I \models \K_1 \land \K'$. Since $\K_1 \land
  \K' \in \TB_\alpha(\K)$ by definition, the claim follows.
The latter case is analogous.
For the converse, we assume that there exists some $\K' \in \TB_\alpha(\K)$
such that $\I \models \K'$.
By definition, this $\K'$ must be of the form $\K_1 \land \K''$
with $\K'' \in \TB_{\alpha_1\cdot\alpha}(\K)$, or of the form $\kbneg \K_1 \land \K''$
with $\K'' \in \TB_{\alpha_2\cdot\alpha}(\K)$. In the former case, it follows
 from the induction hypothesis that $\I \models \K_1 \land \tro_{\alpha_1\cdot\alpha}(\K)$, and hence
 $\I \models \big( \K_1 \land \tro_{\alpha_1\cdot\alpha}(\K) \big ) \lor \big( \kbneg \K_1 \land
 \tro_{\alpha_2\cdot\alpha}(\K) \big )$ and the claim follows. The second case,
where $\K'$ is of the form $\kbneg \K_1 \land \K''$, is analogous to the first
one.
\end{proof}

\begin{proof}[Proof of Theorem~\ref{thm:conp-hardness}]
  We employ the 3-Coloring problem for graphs. Assume a graph $G=(V,E)$ with
  $V=\{1,\ldots,n\}$. We construct in polynomial time a KB $\K$ and an action
  $\alpha$ such that $G$ is 3-colorable iff $\alpha$ is not
  $\K$- preserving. 
  For every $v\in V$, we use 3 concept names
  $A_v^{0},A_v^{1},A_v^{2}$ for the 3 possible colors of the vertex $v$. In
  addition, we employ a concept name $D$. Let $\K$ be the following
  KB:
  \[
    \textstyle
    \K=(D\ISA\neg D) \land
    \bigwedge_{ (v,v')\in E \land 0 \leq c \leq 2} (A_v^{c}\ISA \neg A_{v'}^{c}).
  \]

  It remains to define the action $\alpha$. For this we additionally use a
  nominal $\{o\}$ and fresh concept names $B_1,\ldots,B_n$. We let
  $\alpha:=\alpha_1 \alpha_2^{1} \cdots \alpha_2^{n} \alpha_3$, where
  \begin{enumerate}[\it (i)]
  \item $\alpha_1=(D\oplus \{o\})\cdot ( B_1\oplus \{o\})\cdots (B_n\oplus
    \{o\})$,
  \item $\alpha_2^{i}=(B_i\ominus A_{i}^0 )\cdot (B_i\ominus A_{i}^1)\cdot
    (B_i\ominus A_{i}^3)$, for all $i\in \{1,\ldots,n\}$, and
  \item $\alpha_3=(D\ominus B_1)\cdots (D\ominus B_n)$.
  \end{enumerate}

  Assume $\I$ is a model of $\K$ such that $S_\alpha(\I)\not\models \K$.
  It is possible to show that then $G$ is 3-colorable.

  Suppose $G$ is 3-colorable and a proper coloring of $G$ is given by a
  function $col:V\rightarrow \{0,1,2\}$. Take any interpretation $\I$ with
  $\Delta^{\I}=\{e\}$ and such that
  \begin{inparaenum}[\it (i)]
  \item $\{o\}^{\I}=e$,
  \item $D^{\I}=\emptyset$,
  \item $e\in (A_v^c)^{\I}$ iff $col(v)=c$.
  \end{inparaenum}
  Since $col$ is a proper coloring of $G$, $\I$ is a model of $\K$. As easily
  seen, $S_{\alpha}(\I)\not\models \K $.
\end{proof}

\begin{proof}[Proof of Theorem~\ref{thm:dllite-complexity}]
\np-hardness is immediate (e.g., by a reduction from propositional
satisfiability).
For membership in \np, we define a non-deterministic rewriting procedure that
transforms in polynomial time a \dlhp KB into a \dlh KB. 
We ensure that a \dlhp KB \K is finitely satisfiable iff there exists a
rewriting of $\K$ into a finitely satisfiable \dlh KB. As satisfiability
testing in \dlh is feasible in polynomial time, we obtain 
an \np upper bound for \dlhp.

Assume a \dlhp KB $\K$. The rewriting of $\K$ has two steps: first, we get rid
of the possible occurrences of $\lor$, and then of the complex concepts and
roles in assertions.

Let $P$ be the set of inclusions and assertions of $\K$. Non-deterministically
pick a set $M\subseteq P$ such that $M$ is a model of $\K$, when $\K$ is seen
as a propositional formula over $P$. Let $\K_M=\bigwedge_{\alpha\in M}\alpha
\land \bigwedge_{\alpha'\not\in M}\kbneg\alpha'$. Clearly, $\K$ is finitely
satisfiable iff we can choose an $M$ with $\K_M$ finitely satisfiable.


In the next step, we show how to obtain from $\K_M$ a \dlh KB.  Let $\curlyT$
be the set of inclusions that occur in $\K_M$ and let $\curlyA$ be the set of
assertions and their negations occurring in $\K_M$. Recall that the inclusions
of $\curlyT$ are inclusions of the standard \dlh, but the assertions in
$\curlyA$ may contain complex concepts.  We non-deterministically complete
$\curlyA$ with further assertions to explicate complex concepts and roles. A
\emph{completion} of $\curlyA$ is a $\subseteq$-minimal set $\ap$ of assertions
that is closed under the conditions in Figure~\ref{completion}.
\begin{figure}[t!]
\begin{boxedminipage}{15cm}
\begin{compactenum}[-]
  \item $\curlyA\subseteq \ap$;
  \item for every assertion $\alpha$, $\alpha\not\in\ap$ or
    $\kbneg\alpha\not\in\ap$;
  \item if $o$ is an individual from $\K_M$ and $C_1\ISA C_2\in \curlyT$, then
    $\kbneg (o: C_1) \in \ap$ or $o:C_2 \in \ap$;
  \item if $(o,o')$ are individuals from $\K_M$ and $r_1\ISA r_2\in \curlyT$,
    then $\kbneg ((o,o'): r_1 )\in \ap$ or $(o,o'):r_2 \in \ap$;
  \item if $o:C_1\AND C_2\in \ap$, then $o:C_1 \in \ap$ and $o:C_2 \,{\in}\, \ap$;
  \item if $o:C_1\OR C_2\in \ap$, then $o:C_1 \in \ap$ or $o:C_2 \in \ap$;
  \item if $o:\exists r.\top\in \ap$, then $(o,o'):r\in \ap$ for a fresh
    $o'$;

  \item if $o:\neg C \in \ap$, then $\kbneg (o: C) \in \ap$;
  \item if $\kbneg (o:C)\in \ap$, then $o:\neg C\in\ap$;

  \item if $o:\neg \neg C \in \ap$, then $o: C\in \ap$;
  \item if $o:\neg (C_1\AND C_2) \in \ap$, then $\kbneg (o: C_1)\in \ap$ or
    $\kbneg (o: C_2)\in \ap$;
  \item if $o:\neg (C_1\OR C_2) \in \ap$, then $\kbneg (o: C_1)\in \ap$ and
    $\kbneg (o: C_2)\in \ap$;
  \item if $o:\neg (\exists r.\top)\in\ap$, then $\kbneg ((o,o'):r\in \ap)$
    for all individuals $o'$ of $\ap$;
  \item if $(o,o'):r \in \ap$, then $(o',o):r^{-}\in \ap$;
  \item if $(o,o'):r_1\cup r_2\in \ap$, then $(o,o'):r_1\in \ap$ or
    $(o,o'):r_2 \in \ap$;
  \item if $(o,o'):r_1\setminus r_2\in \ap$, then $(o,o'):r_1\in \ap$ and
    $\kbneg ((o,o'):r_2) \in \ap$;
  \item if $\kbneg ((o,o'):r_1\cup r_2)\in \ap$, then $\kbneg
    ((o,o'):r_1)\in\ap$ and $\kbneg ((o,o'): r_2)\in\ap$;
  \item if $\kbneg ((o,o'):r_1 \setminus r_2)\in \ap$, then $\kbneg
    ((o,o'):r_1)\in \ap$ or $(o,o'):r_2 \in \ap$;
  \item if $o:\{o'\}\in \ap$, then $o=o'$;
  \item if $(o_1,o_2):\{(o_1',o_2')\}\in \ap$, then $o_1=o_1'$ and $o_2=o_2'$;
  \end{compactenum}
\end{boxedminipage}
\caption{Completion for \dlhp ABoxes \label{completion}}
\end{figure}
Let $\ap_{b}$ be the restriction of $\ap$ to basic assertions. Clearly,
$\bigwedge \curlyT\land \bigwedge \ap_{b} $ is a \dlh KB. It is not difficult
to see that $\K_M$ is finitely satisfiable iff there exists a completion $\ap$
such that $\bigwedge \curlyT\land \bigwedge \ap_{b} $ is finitely satisfiable.
\end{proof}

\begin{proof}[Proof of Theorem~\ref{thm:reduction-dllite}]
The lower bound follows from Theorem~\ref{thm:conp-hardness}, or
alternatively, it can be proved by a reduction from finite unsatisfiability in
\dlhp, employing the same reduction as in the proof of
Theorem~\ref{thm:verification-nexp}.

\added[MO]
{For the upper bound, assume a \dlhp KB $\K$ and a simple action $\alpha$.
We proceed analogously to the Proof of \ref{thm:verification-nexp}.
From Theorem~\ref{thm:verification} we know that
$\alpha$ is not $\K$-preserving iff
$\K\land \kbneg \STR_{\alpha^{*}}(\K)$ is finitely satisfiable. Moreover,
we have shown that $\K\land \kbneg \STR_{\alpha^{*}}(\K)$ is finitely
satisfiable iff there exists a $\K' \in \TB_{\alpha^{*}}(\K)$ such that $\K
\land \K'$ is not finitely satisfiable, and $\K'$
can be obtained non-deterministically in polynomial time and is of size
polynomial in $\alpha$ and $\K$.
The KB $\K'$
is not a \dlhp KB, but it can be transformed into an equisatisfiable \dlhp KB
in linear time. To this end, turn $\K'$ into negation normal form, i.e., push
$\kbneg$ inside so that $\kbneg$ occurs in front of inclusions and assertions
only. Then replace every occurrence of $\kbneg (B_1\ISA B_2)$ and $\kbneg
(r_1\ISA r_2)$ in the resulting $\K'$ by $o: B_1\AND \neg B_2$ and $(o,o'):
r_1\setminus r_2$, respectively, where $o,o'$ are fresh individuals. Clearly,
the above transformations preserve satisfiability. Moreover, since in $\K$ the
operator $\kbneg$ may occur only in front of assertions, and $\alpha$ is
simple, every inclusion in the resulting $\K'$ already appears in $\K$. This
implies that $\K'$ is a \dlhp KB as desired.}
\end{proof}

\begin{proof}[Proof of Theorem~\ref{thm:plan-existence-undecidable}]
The proof is by reduction from the Halting problem. We reduce to (P1) and to
(P2) deciding whether a deterministic Turing machine $M$ accepts a word $w \in
\{0,1\}^*$.

For (P1), assume $M$ is given by a tuple $M=(Q,\delta,q_0,q_a,q_r)$, where $Q$
is a set of states, $\delta: \{0,1,b\}\times Q \rightarrow \{0,1,b\}\times
Q\times \{+1,-1\}$ is the transition function, $b$ is the blank symbol,
$q_0\in Q$ is the initial state, $q_a\in Q$ is the accepting state, and
$q_r\in Q$ is the rejecting state. We can assume w.l.o.g.\ that after
accepting or rejecting the input the machine returns the read/write head to
the initial position.

Intuitively, we define an action that implements the effects of each possible
transition from $\delta$.  We also have a pair of actions that ``extend'' the
tape with blank symbols as needed. For the reduction we use the role $next$,
concept names $Sym_0,Sym_1,Sym_b$, and $St_q$ for each $q\in Q$.

The set $\Act$ of actions is defined as follows.  For every $(\sigma,q)\in
\{0,1,b\}\times Q$ with $\delta(\sigma,q)=(\sigma',q',D)$ we have the action
$\alpha_{\sigma,q}=(x_1,x_2):next \land x_2:Sym_\sigma\land x_2:St_q \land
(x_2,x_3):next \,? \, (Sym_\sigma \ominus \{x_2\} )( Sym_{\sigma'} \oplus
\{x_2\} )( St_q \ominus \{x_2\} )(St_{q'} \oplus \{x_{2+D}\})$. To extend the
tape with blank symbols, we have the actions $\alpha_r$ and $\alpha_l$. In
particular, $\alpha_r= x:(Sym_0\OR Sym_1 \OR Sym_b) \land y:\neg (Sym_0\OR
Sym_1 \OR Sym_b) ? (next \oplus \{(x,y)\})(Sym_b \oplus \{y\})\}$. The action
$\alpha_l$ is obtained from $\alpha_r$ by replacing $(next \oplus \{(x,y)\})$
with $(next \oplus \{(y,x)\})$. We finally have an initialization action
$\alpha_{init} $ which stores the initial configuration of $M$ in the
database.  In particular, $\alpha_{init} = (a_1: \neg (Sym_0 \OR Sym_1 \OR
Sym_b) ) ? (Sym_{\sigma_1}\oplus \{a_1\} )\cdots (Sym_{\sigma_m}\oplus \{a_m\}
) (St_{q_0} \oplus \{a_1\})$, where $ \sigma_1\cdots \sigma_m = w$. We let $
\K = a_1\,{:}\,St_{q_a}\OR St_{q_r}$  and the initial database $\I$ is empty,
i.e. no domain element participates in a concept or a role.

It can be easily seen that the reduction is correct.  If $\K$ has a plan, then
$M$ halts on $w$. Conversely, if $M$ halts on $w$, then it halts within some
number of steps $s$. One can verify that expanding the domain of $\I$ with $s$
fresh elements is sufficient to find  a plan for $\K$ using the actions in $\Act$.

The above reduction also applies to (P2). It suffices to define a precondition
KB $\K_\mathit{pre}$ that describes the above $\I$. Simply let $\K_\mathit{pre}$ be the
conjunction of $(Sym_0\OR Sym_1 \OR Sym_b\OR\exists next \OR \exists next^{-}
\ISA \bot)$ and $ \bigsqcup_{q\in Q}St_q\ISA \bot$.
\end{proof}

\begin{proof}[Proof of Theorem~\ref{thm:finite-plan-existence-p1-pspace}]
The lower bound can be proven by an easy adaption of the reduction in
Theorem~\ref{thm:plan-existence-undecidable}.

For the upper bound we employ a non-deterministic polynomial space procedure
that stores in memory a finite interpretation and non-deterministically applies
actions until the goal is satisfied. Since the domain of each candidate
interpretation is fixed and of size linear in the input, each of them can be
represented in polynomial space. The number of possible interpretations is
bounded by $c=2^{r \cdot d^2 + c \cdot d}$, where $r$ and $c$ are respectively
the number of concepts and roles appearing in the input set of actions, and $d$
is the cardinality of the domain of the initial interpretation. Thus the
procedure can be terminated after $c$ many steps, without loss of completeness.
We note that a counter that counts up to $c$ can be implemented in polynomial
space, and that model checking \OURDL-formulae is feasible in polynomial space.
\end{proof}

\begin{proof}[Proof of Theorem~\ref{thm:finite-plan-existence-p2-nexptime}]
The lower bounds can be immediately inferred from the complexity of static
verification with KBs in  \OURDL (Theorem~\ref{thm:verification-nexp})  and
\dlhp (Theorem~\ref{thm:conp-hardness}).

For the upper bounds, 
we first guess a variable substitution $\sigma$ and a sequence
$P=\tuple{\alpha_1,\ldots,\alpha_m}$ of at most $k$ actions.  By
Theorem~\ref{thm:regression}, it follows that $P$ is a plan as desired iff
$\sigma(\K_\mathit{pre})\land \STR_{\alpha_1 \cdots \alpha_m}(\sigma(\K))$
is finitely satisfiable.
\added[MO]{To be able to check the finite satisfiability of
$\sigma(\K_\mathit{pre})\land \STR_{\alpha_1 \cdots \alpha_m}(\sigma(\K))$
within the desired bounds, we proceed similarly as above, and  consider a procedure that
non-deterministically builds a polynomial $\K'$ such that
$\sigma(\K_\mathit{pre})\land \K'$ is finitely satisfiable iff
$\sigma(\K_\mathit{pre})\land \STR_{\alpha_1 \cdots \alpha_m}(\sigma(\K))$ is
finitely satisfiable. }\todo{the rest of the proof is new}
Note that the core difference between this proof and the ones of
Theorems~\ref{thm:verification-nexp} and~\ref{thm:reduction-dllite} is that
now the formula $\STR_{\alpha_1 \cdots \alpha_m}(\sigma(\K))$ is not negated
and hence, intuitively, we need to decide the existence of
an interpretation that satisfies the negation of all formulas  in
$\TB_{\alpha}(\K)$, rather than satisfying just one of them.

We define a set of KBs $\TBC_\alpha(\K)$ that is similar to
$\TB_\alpha(\K)$, but contains the negation of the formulas in the latter,
and uses conjunction rather than implications for the conditional axioms.
\begin{align*}
  \TBC_\varepsilon(\K)= &  \{ \K \} \\
   \TBC_{(A \oplus C)\cdot
     \alpha}(\K)= &  \{ \K'_{A\gets A\OR C} \mid \K' \in \TBC_{\alpha}(\K)\} \\
   \TBC_{(A \ominus C)\cdot
     \alpha}(\K)= &  \{ \K'_{A\gets A\AND \neg C} \mid \K' \in \TBC_{\alpha}(\K)\} \\
   \TBC_{(p \oplus r)\cdot
     \alpha}(\K)= &  \{ \K'_{p\gets p\cup r} \mid \K' \in
   \TBC_{\alpha}(\K) \}  \\
   \TBC_{(p \ominus r)\cdot
     \alpha}(\K)= & \{ \K'_{p\gets p\setminus r} \mid \K' \in \TBC_{\alpha}(\K)\} \\
   \TBC_{(\cact{\K_1}{\alpha_1}{\alpha_2})\cdot
        \alpha}(\K) = &  \{ \K_1 \,{\land}\, \K' \mid \K' \in
      \TBC_{\alpha_1\cdot\alpha}(\K)  \} \cup \{ \kbneg \K_1 \,{\land}\, \K'
      \mid \K' \in \TBC_{\alpha_2\cdot\alpha}(\K) \}
\end{align*}
Similarly as above, $|\TBC_\alpha(\K)|$ may be
exponential but each $\K' \in \TBC_\alpha(\K)$ is polynomial and can be built
non-deterministically in polynomial time.
We show below the following claim:
\begin{compactitem}
\item [$(\ddagger)$]
For every $\I$ and every $\K$, there exists some $\K' \in
\TBC_{\alpha}(\K)$ such that $\I \models \K'$ iff   $\I \models \STR_{\alpha}(\K)$.
\end{compactitem}
With $(\ddagger)$ we can easily show that
$\sigma(\K_\mathit{pre})\land \STR_{\alpha_1 \cdots \alpha_m}(\sigma(\K))$ is
finitely satisfiable iff there exists some  $\K' \in \TBC_{\alpha_1 \cdots \alpha_m}(\sigma(\K))$ such that
$\sigma(\K_\mathit{pre})\land \K'$ is finitely satisfiable.
For the `only if' direction, assume $\sigma(\K_\mathit{pre})\land \STR_{\alpha_1 \cdots \alpha_m}(\sigma(\K))$ is
finitely satisfiable. Then there exists some finite $\I$ such that $\I \models
\sigma(\K_\mathit{pre})$ and $\I \models \STR_{\alpha_1 \cdots
  \alpha_m}(\sigma(\K))$. By $(\ddagger)$, for this $\I$ there is some $\K' \in
\TBC_{\alpha_1 \cdots \alpha_m}(\sigma(\K))$ such that $\I \models \K'$ iff $\I \models \STR_{\alpha_1 \cdots
  \alpha_m}(\sigma(\K))$. We choose this $\K'$. It follows that $\I \models \K'$ and, since $\I \models
\sigma(\K_\mathit{pre})$, we can conclude that $\sigma(\K_\mathit{pre})\land \K'$ is finitely satisfiable.
For the other direction, assume that there is no  $\K' \in \TBC_{\alpha_1 \cdots \alpha_m}(\sigma(\K))$ such that
$\sigma(\K_\mathit{pre})\land \K'$ is finitely satisfiable. Then it follows
that: $(*)$ $\I
\not\models \K'$ for every $\K' \in \TBC_{\alpha_1 \cdots
  \alpha_m}(\sigma(\K))$ and every $\I$ with $\I \models
\sigma(\K_\mathit{pre})$. Assume towards a contradiction that
$\sigma(\K_\mathit{pre})\land \STR_{\alpha_1 \cdots \alpha_m}(\sigma(\K))$ is
 satisfiable. Then there is some $\I$ with  $\I \models
\sigma(\K_\mathit{pre})$ and $\I \models \STR_{\alpha_1 \cdots \alpha_m}(\sigma(\K))$,
and by $(\ddagger)$,
for this $\I$ there is some $\K' \in \TBC_{\alpha_1 \cdots
  \alpha_m}(\sigma(\K))$ such that $\I \models \K'$  iff
$\I \models \STR_{\alpha_1 \cdots \alpha_m}(\sigma(\K))$. This would imply that
$\I \models \K'$, contradicting $(*)$.
Having shown this, the upper bound follows directly from the complexity of deciding finite
satisfiability of $\sigma(\K_\mathit{pre})\land \K'$, and the fact that $\K'$
is of polynomial size and can be obtained non-deterministically in polynomial
time.

It is only left to show $(\ddagger)$, what we do by induction on $s(\alpha)$.
The base case is trivial, since for $\alpha =\varepsilon$ we have
$\TBC_{\alpha}(\K) = \{ \K \}$ and $\STR_{\alpha}(\K) = \K$,
so we can set $\K' = \K$ and the claim follows.

For the case of $\alpha= A \oplus C \cdot \alpha'$, we have
$ \STR_{\alpha'}(\K)  =  \STR_{\alpha'}(\K)$.
By induction hypothesis
there is some $\K'' \in \TBC_{\alpha'}(\K)$ such that $\I \models \K''$
iff $\I \models \STR_{\alpha'}(\K)$. We let $\K' = \K''_{A \gets A \OR
  C}$. Then $\K' \in \TBC_{\alpha}(\K)$, and $\I  \models \K'$
iff $\I \models \STR_{\alpha'}(\K)_{A \gets A \OR C}$
as desired.
The cases of  $\alpha  = (A \ominus C) \cdot \alpha'$,   $\alpha  = (p \oplus
r) \cdot \alpha'$,  and   $\alpha  = (p \ominus r) \cdot \alpha'$ are
analogous.

Finally, if  $\alpha= (\cact{\K_1}{\alpha_1}{\alpha_2}) \cdot \alpha'$, the
choice of $\K'$ depends on $\I$. We distinguish two cases:
\begin{compactitem}
\item
If $\I \models \K_1$, let $\K'' \in
\TBC_{\alpha_1\cdot\alpha'}(\K)$  be such that $\I \models \K''$ iff $\I \models
\STR_{\alpha_1\cdot\alpha'}(\K)$, which exists the induction hypothesis.
Then we set $\K' = \K_1 \land \K''$. We have  $\K'' \in
\TBC_{\alpha}(\K)$ by definition. Now we show that $\I \models \K'$ iff $\I
\models \STR_{\alpha}(\K)$.

 Assume $\I \models \K'$.
Then $\I \models \K''$,
and $\I \models \STR_{\alpha_1\cdot\alpha'}(\K)$.
This ensures that $\I \models \kbneg\K_1
\,{\lor}\,\STR_{\alpha_1\cdot\alpha'}(\K)$.
Since $\I \models \K_1$, we also have $\I \models \K_1
\,{\lor}\,\STR_{\alpha_2\cdot\alpha}(\K)$.
Since $\STR_{\alpha}(\K)  =
      (\kbneg\K_1 \,{\lor}\,\STR_{\alpha_1\cdot\alpha'}(\K)) \land {}
       (\K_1 \,{\lor}\,\STR_{\alpha_2\cdot\alpha'}(\K))$, we obtain $\I \models
       \STR_{(\cact{\K_1}{\alpha_1}{\alpha_2})\cdot
        \alpha'}(\K)$ as desired.

For the converse, assume $\I
\models \STR_{\alpha}(\K)$, that is,
$\I
\models  \kbneg\K_1 \,{\lor}\,\STR_{\alpha_1\cdot\alpha'}(\K)$ and
$\I
\models \K_1 \,{\lor}\,\STR_{\alpha_2\cdot\alpha'}(\K)$.
From the former and $\I \models \K_1$, it follows that
$\I \models  \STR_{\alpha_1\cdot\alpha'}(\K)$. By our selection of $\K''$,
this implies $\I \models \K''$, and we also have that $\I \models \K_1$, we
can conclude $\I \models \K'$ as desired.
\item Otherwise,  if $\I \models \kbneg \K_1$,
let $\K''$ be such that  $\K'' \in
\TBC_{\alpha_2\cdot\alpha'}(\K)$ and $\I \models \K''$ iff $\I \models
\STR_{\alpha_2\cdot\alpha'}(\K)$ (such a $\K''$ exists by the induction hypothesis), and let
$\K' = \kbneg \K_1 \land \K''$. Then $\K'' \in
\TBC_{\alpha}(\K)$, and the  proof of $\I \models \K'$ iff $\I
\models \STR_{\alpha}(\K)$ is analogous to the first case.
\end{compactitem}\end{proof}

\todo{check if problem names in appendix are consistent with body of the paper}

\begin{proof}[Proof of Theorem~\ref{thm:planning-complexity}]
Problem (S) can be shown to be undecidable by employing the same reduction as
for (P2) in~Theorem~\ref{thm:plan-existence-undecidable}.
The co\nexptime lower bounds for (C) and (S\textsubscript{b}) trivially follow from finite
  satisfiability in \OURDL.

\added[MO]{For the upper bounds, we first observe that (C) reduces to validity testing
in \OURDL: an instance of (C) (as described above) is positive iff the formula
$\sigma(\K_\mathit{pre}')\rightarrow \STR_{\alpha_1\cdots\alpha_n}(\sigma(\K'))$ is
 valid, where $\K_\mathit{pre}',\K'$ are obtained from $\K_\mathit{pre},\K$ by replacing
 every variable by a fresh individual.
Deciding validity of $\sigma(\K_\mathit{pre}')\rightarrow
\STR_{\alpha_1\cdots\alpha_n}(\sigma(\K'))$ in turn reduces to deciding
whether $\sigma(\K_\mathit{pre}')\land \kbneg
\STR_{\alpha_1\cdots\alpha_n}(\sigma(\K'))$ is finitely unsatisfiable.
The upper bounds for (C) then follow from the \np and \nexptime upperbounds for the
satisfiability of KBs of the form $\K' \land \kbneg \STR_{\alpha}(\K)$ shown
in the proofs of Theorems~\ref{thm:verification-nexp}
and~\ref{thm:reduction-dllite}.}

Negative instances of (S\textsubscript{b}),
where $\K_\mathit{pre}$ is the precondition and $\K$ is the goal, can be recognized in
\nexptime. Such a test comprises building an exponentially large set of all
candidate action sequences of length at most $k$, and then making sure that
that each candidate is invalidated. That is, each candidate action
sequence $P$ induces an instance of (C), which can be shown
negative in \nexptime.
In the case of \dlhp and simple actions, 
we can guess non-deterministically a sequence of actions  of length at most
$k$ and then check that the induced instance of (C) is positive, which is a
test in co\np. It is not difficult to see that the $\np^\np$ upper bound is
tight. This can be shown by a polynomial time reduction from evaluating QBFs
of the form $\gamma=\exists p_1 \ldots \exists p_n \forall q_1 \ldots \forall
q_m. \psi$, where $\psi$ is a Boolean combination over propositional variables
$V=\{p_1, \ldots , p_n, q_1, \ldots ,q_m\}$. We can assume that negation in
$\psi$ occurs in front of propositional variables only.  For the reduction to
(S\textsubscript{b}), we employ concept names $T$ and $F$, and individual names $o_v$ for each
propositional variable $v\in V$. We let $\K_\mathit{pre}= \big ( \bigwedge_{1 \leq i \leq n
} o_{p_i}\,{:}\,\neg ( T \OR F ) \big) \land \big( \bigwedge_{1 \leq i \leq m } o_{q_i}\,{:}\,( T
\OR F ) \AND (\neg T \OR \neg F ) \big )$.  Intuitively, each initial
interpretation encodes an assignment for the variables $q_1, \ldots , q_m$,
but does not say anything about $p_1, \ldots , p_n$.
The latter is determined by choosing a
candidate plan. To this end, for each $1\leq i \leq n$, we construct the
following actions:
\[
  \alpha_{i}= o_{p_i}\,{:}\,\neg F \,?\, T\oplus
  \{o_{p_i}\},
  \qquad
  \alpha_{i}'= o_{p_i}\,{:}\,\neg T \,?\, F\oplus \{o_{p_i}\}.
\]
We
finally let $k=n$ and let $\K$ be the KB obtained from $\psi$ by replacing
each negative literal $\neg v$ by $o_v\,{:}\,F$ and each positive literal $v$ by
$o_v\,{:}\,T$. It is not difficult to see that $\gamma$ evaluates to \emph{true}
iff the constructed instance of (S\textsubscript{b}) is positive.
\end{proof}


\end{document}